\documentclass[11pt]{article}

\usepackage{bbold,bm,amsmath,amssymb,latexsym}
\usepackage[colorlinks = true, citecolor = blue]{hyperref}
\usepackage{bookmark}
\usepackage[numbers]{natbib}
\usepackage[all]{hypcap}
\usepackage{booktabs}
\usepackage{cite}{\tiny }
\usepackage{afterpage}
\usepackage{graphicx}
\usepackage{mdframed}
\usepackage[normalem]{ulem}
\usepackage[table]{xcolor}
\usepackage{makeidx}
\usepackage{cleveref}
\usepackage[centerlast]{caption}
\usepackage{float}
\usepackage[lined,boxed,commentsnumbered]{algorithm2e}
\usepackage[all]{xy}
\usepackage{tikz}
\usepackage{pgfplots}
\usetikzlibrary{arrows,intersections}
\usepackage{times}
\usepackage{float}
\usepackage{subfig}
\usepackage{authblk}
\usepackage{}
\usepackage[T1]{fontenc}

\newlength\aftertitskip     \newlength\beforetitskip
\newlength\interauthorskip  \newlength\aftermaketitskip

\newcommand{\BlackBox}{\rule{1.5ex}{1.5ex}}  
\newenvironment{proof}{\par\noindent{\bf Proof\ }}{\hfill\BlackBox\\[2mm]}
 
\newtheorem{theorem}{Theorem}
\newtheorem{lemma}[theorem]{Lemma}

\newtheorem{corollary}[theorem]{Corollary}
\newtheorem{definition}[theorem]{Definition}

\DeclareMathOperator*{\argmin}{arg\,min}

\newcommand{\RR}{\protect\mathbb{R}}

\newcommand{\EE}{\protect\mathbb{E}}
\newcommand{\PP}{\protect\mathbb{P}}
\newcommand{\TT}{\protect\mathbb{T}}
\newcommand{\MM}{\protect\mathbb{M}}
\newcommand{\obs}{\protect\mathcal{Z}}

\newcommand{\ind}{\protect\bm{1}}

\newcommand{\dist}{\raise.17ex\hbox{$\scriptstyle\mathtt{\sim}$}}
\newcommand{\Lbar}{\protect{\underline{L}}}
\newcommand{\risk}{\protect{\mathrm{Risk}}}
\newcommand{\Rbar}[1]{\underline{\risk}_{#1}}

\newcommand{\Lagrange}{\protect{\mathcal{L}}}
\newcommand{\nL}{\protect{\nabla \Lbar}}

\newcommand{\norm}[1]{\left\lVert #1 \right\rVert}

\newcommand{\ip}[2]{\left< #1, #2\right>}

\newcommand{\dual}[1]{(\RR^{#1})^*}

\newcommand{\riskL}[1]{\risk_{#1}}
\newcommand{\pred}[1]{[\![ #1 ]\!]}

\newcommand{\id}[1]{\mathrm{id}_{#1}}

\pgfplotsset{compat=1.18}

\title{On Experiments}

\author[1]{Brendan van Rooyen}
\affil[1]{Decisions 360}

\date{}

\begin{document}
	
\maketitle

\begin{abstract}
	The scientific process is a means to turn the results of experiments into knowledge about the world in which we live. Much research effort has been directed toward \emph{automating} this process. To do this, one needs to \emph{formulate} the scientific process in a \emph{precise} mathematical language. This paper outlines one such language. What is presented here is hardly new. The material is based on great thinkers from times past \citep{VonNeumann1945,Wald1949,Blackwell1951,DeGroot1962,Le1964} as well as more modern contributions \citep{Cam2011,Torgersen1991,Reid2009b,Grunwald2004,Dawid2007}. The novel contributions of this paper are:
	
	\begin{enumerate}
		\item A new general data processing inequality.
		\item A bias variance decomposition for canonical losses.
		\item Streamlined proofs of the Blackwell-Sherman-Stein and Randomization theorems.
		\item Means of calculating deficiency through linear programming.
	\end{enumerate}	
\end{abstract}

\section{Introduction}

The notion of an experiment is a key component of the scientific process. Experiments link the values of some unknown property of the world to the outcomes of the experiment. Performing an experiment provides \emph{information}. We construct a general language for defining experiments, the information contained therein, and means to qualitatively and quantitatively compare experiments. Information will be seen to be governed by transformation, the information present explained by those transformations that render it untouched. Mathematically, the study of transformations is a subset of category theory. Therefore, we make allusions to this theory where appropriate, showing where some constructions in statistics can be seen as general constructions in category theory.
\\
\\
These ideas are hardly new, and we lean heavily on great thinkers of the past. We argue that \emph{presentation} of the ideas is novel, modern, and concise. Of course, we also offer some new results. The novel contributions of this paper are:

\begin{enumerate}
	\item A new general data-processing inequality.
	\item A bias variance decomposition for canonical losses.
	\item Streamlined proofs of the Blackwell-Sherman-Stein and Randomization Theorems.
	\item Means to calculate deficiency via linear programming.
\end{enumerate}
While our ultimate goal is the comparison of experiments, in section \ref{Optimal Decisions} we also provide results on the representation of loss functions. The results in this section have appeared previously, as Appendix C of \citep{JMLR:v18:16-315}. We include them here for completeness.

\section{The General Decision Problem}

We phrase the general decision problem as a two-player game between nature and a decision maker. Let $\Theta$ be a set of possible values of some unknown, and $A$ the set of actions available to the decision maker. The consequence of an action is measured by a loss function $L : \Theta \times A \rightarrow \RR$. A negative loss represents a gain to the decision maker. The norm of a loss function is given by its largest possible consequence (positive or negative), $\norm{L}_\infty = \max_{\theta, a} \lvert L(\theta,a) \rvert$.
\\
\\
Unknown to the decision maker is the \emph{exact} value of $\theta$. To \emph{discover} this, the decision maker is guided by experiments. An experiment is a special kind of relationship between the set of unknowns and the outcome of the experiment. Due to limitations in the type/number of tests available to the decision maker, as well as potential errors in the measurement apparatus, the exact identity of $\theta$ can remain a mystery despite observed data. We assume that the decision maker represents their uncertainty in $\theta$ through probability.

\section{Probability}

The loss function itself does not tell us how to act; it provides a \emph{loss profile}, a function $L_a : \Theta \rightarrow \RR$, for each action $a \in A$. This loss profile summarizes the loss incurred for each unknown value for a fixed action. Our knowledge of $\theta$ allows us to place a value on each of these actions, through an \emph{expectation}, $$ \EE: \RR^\Theta \rightarrow \RR,
$$
which generally can be \emph{any} function from loss profiles to reals. Ultimately, the decision maker chooses the action that they think has the lowest expectation. Exactly how one maps loss profiles to their expectations is up to the individual. Here we place the restriction that the expectation should be given by a \emph{probability distribution}.

\begin{definition}\label{Probability Distribution}
	
	A \emph{probability distribution} on a set $X$ is an element of $\dual{X}$, i.e. a linear function $\EE_P : \RR^{X} \rightarrow \RR$, such that:
	
	\begin{enumerate}
		
		\item $\EE_P \ind_X = 1$.
		\item If $f(x) \leq g(x) ,\ \forall x \in X$ then $\EE_P f \leq \EE_P g$.
		
	\end{enumerate}
	
\end{definition}
Point one states that probability distributions are \emph{normalized}, point two that they are \emph{positive}. At times we use infix notation, with $\EE_{P} f = \ip{P}{f}$. Define the set of all distributions on a set $X$ to be $\PP(X)$. For any $x \in X$ define the point mass distribution $\delta_x$, with $\EE_{\delta_x} f = f(x)$ for all functions $f$.

\subsection{Measure Theoretic Technicalities}

The language as developed is fine for dealing with finite sets, where elements of the dual are merely finite dimension vectors. When dealing with infinite sets, the issue of how to \emph{represent} elements of $\left(\RR^X\right)^*$ becomes more troublesome. The usual route is to only consider loss profiles that are \emph{measurable}. Rather than all functions $\RR^X$ we work with a subalgebra of measurable functions, which we call $\mathrm{Meas}(X, \RR)$. A probability distribution is then a positive, normalized, linear function $\mathrm{Meas}(X,\RR) \rightarrow \RR$. Restricting oneself in this way affords the \emph{representation} of probability distributions via the Lebesgue integral. Of course, one could take our work as pertaining to $\mathrm{Meas}(X,\RR)$ rather than all of $\RR^X$, without any real loss in the results. Our primary interest is when:
$$
\text{The Basic Sets of Study are Finite.}
$$
We argue that this is a defensible position to take if one is interested in computer science. After all, a digital computer functions via the manipulation of a finite set of states, meaning the objects we are interested in can either be represented via a finite set or can be sufficiently approximated by one.
\\
\\
At times, the natural set of actions will \emph{not} be a finite set, but rather a well-behaved subset of $\RR^n$. 

\section{Optimal Decisions}\label{Optimal Decisions}

Returning to our decision maker, we assume that they represent their uncertainty in $\Theta$, via a probability distribution
$P \in \PP(\Theta)$, and that they choose their actions minimizing the expected loss.

\subsection{Representing Loss Functions}

Here we develop \emph{general} representations of loss functions $L : \Theta \times A \rightarrow \RR$. We assume that the sets $\Theta$ and $A$ are finite. Even if one starts with a finite set of actions, in many statistical problems it is natural to consider distributions over unknowns $\PP(\Theta)$.

\begin{definition}
	
	A proper loss is a function $L : \Theta \times \PP(\Theta) \rightarrow \RR$ such that $\forall P \in \PP(\Theta)$,
	$$
	P \in \argmin_{Q \in \PP(\Theta)} \EE_P L_Q.
	$$
	
\end{definition}
A proper loss takes a prediction $Q \in \PP(\Theta)$, and then penalizes the decision maker according to the amount of weight that their prediction assigns to the unknown $\theta$. Properness ensures that if the decision maker \emph{knows} $P$, then they minimize their expected loss by \emph{reporting} $P$. Proper losses constitute a well-studied class of loss functions that provide suitable surrogates for decision problems \citep{brier1950verification,Grunwald2004,Zhang2004,Gneiting2007,Dawid2007,Reid2009a,Dawid2007,AvilaPires2013}.
\\
\\
Furthermore, we show how to render any proper loss convex through a canonical reparameterization. This allows the use of tools from convex analysis \citep{Boyd2004,Lucchetti2006} to aid in calculating optimal actions. 

\subsection{Entropy from Loss}

Rather than working with probability distributions, we take the route of \citep{Williamson2014} and work with unnormalized distributions. Denote the set of all unnormalized distributions on $\Theta$ by $\PP^+(\Theta)$. For any loss function $L$, define the \emph{entropy} $\Lbar : \PP^+(\Theta) \rightarrow \RR$, 
$$
\Lbar(\mu) = \min_{a \in A} \langle \mu, L_a \rangle.
$$ 
$\Lbar(P)$ measures the uncertainty of the optimal action for the distribution $P$. The entropy is also called an \emph{uncertainty function}, a \emph{Bayes risk} or a \emph{support function} \citep{DeGroot1962,Williamson2014}. It is concave and 1-homogeneous.

\begin{definition}
	A function $f : \PP^+(\Theta)\rightarrow \RR$ is 1-homogeneous if for all $\mu \in \PP^+(\Theta)$ and for all $\lambda > 0$, 
	$$
	f(\lambda \mu) = \lambda f(\mu).
	$$
	
\end{definition}

\subsection{Loss from Entropy}\label{sec:loss-from-entropy}

All loss functions give rise to an entropy. The entropy encodes much information about its associated loss through its \emph{super-gradients}, which includes all the \emph{Bayes} actions for the underlying loss. For any distribution $P$, define the \emph{Bayes actions} for $P$ as the set of minimizers,
$$
A_P = \argmin_{a \in A} \ip{P}{L_a}.
$$
For any $a_P \in A_P$ we have $\Lbar(P) = \ip{P}{L_{a_P}}$. 

\begin{definition}[Super-gradient of a concave function]
	Let $f : \PP^+(\Theta)\rightarrow \RR$ be a concave function. $v \in \RR^\Theta$ is a \emph{super-gradient} of $f$ at the point $x$ if for all $y\in \PP^+(\Theta)$,
	$$
	\ip{y-x}{v} + f(x) \geq f(y).
	$$
\end{definition}
Denote the set of all supergradients at a point $x$ by $\partial f(x)$, and the set of all supergradients by $\partial f = \cup_{x} \partial f(x)$. For differentiable concave functions, the supergradients are the same as regular gradients \citep{Lucchetti2006}. 1-homogeneous functions afford a very simple representation via their super-gradients.

\begin{theorem}[Generalized Euler's Homogeneous Function Theorem]\label{Eulers Homo Theorem}
	Let $f : \PP^+(\Theta) \rightarrow \RR$ be a concave 1-homogeneous function. Then for all $x$ and for all $v \in \partial f(x)$,
	$$
	f(x) = \langle x, v \rangle.
	$$
	Furthermore, $v \in \partial f (x) \implies v \in \partial f (\lambda x)$ for all $\lambda > 0$.
\end{theorem}
We include a simple proof of this theorem for completeness.

\begin{proof}
	Firstly, for all $x$,
	$$
	\ip{\frac{1}{2} x - x}{v} + f(x) \geq \frac{1}{2} f(x),
	$$
	which follows directly from the definition of a supergradient at $x$ and the 1-homogeneity of $f$. Rearranging yields $\frac{1}{2} (f(x) - \ip{x}{v}) \geq 0$, which implies $f(x) \geq \ip{x}{v}$. Similarly,
	$$
	\ip{x - \frac{1}{2} x}{v} + \frac{1}{2}f(x) \geq f(x),
	$$
	which follows directly from the definition of a supergradient at $\lambda x$ and the 1-homogeneity of $f$. Rearranging yields $\frac{1}{2} (f(x) - \ip{x}{v}) \leq 0$, yielding $f(x) \leq \ip{x}{v}$. Combing the results gives $f(x) = \ip{x}{v}$.
	To prove the second claim, we have for all $y$ and $\lambda >0$,
	\begin{align*}
	&\ip{y-x}{v} + f(x) \geq f(y) \\
	&\ip{\lambda y - \lambda x}{v} + f(\lambda x) \geq f(\lambda y),
	\end{align*}
	where the first line is by definition and the second is by 1-homogeneity. As $y$ is arbitrary, the claim is proved.

\end{proof}
This theorem provides a corollary that shows the supergradients of a 1-homogeneous function have a property similar to properness.

\begin{corollary}\label{proper super gradients}
	Let $f : \PP^+(\Theta) \rightarrow \RR$ be a concave 1-homogeneous function. Then for all $x,y \in \PP^+(\Theta)$ and for all $v_x \in \partial f(x)$, $v_y \in \partial f(y)$,
	$$
	\ip{x}{v_y} \geq \ip{x}{v_x}. 
	$$
	
\end{corollary}
We now show that the partial loss of a Bayes action is a supergradient of $\Lbar$.

\begin{theorem}
	For all loss functions $L$ and distributions $P$, $a_P \in A_P \Leftrightarrow L_{a_P} \in \partial \Lbar(P)$.
\end{theorem}

\begin{proof}
	For $a_P \in A_P$ we have for all $\mu \in \PP^+(\Theta)$,
	$$
	\ip{\mu - P}{L_{a_P}} + \Lbar(P) = \ip{\mu}{L_{a_P}} \geq \min_{a \in A} \ip{\mu}{L_a} = \Lbar(\mu).
	$$
	Hence $L_{a_P} \in \partial \Lbar(P)$. For the converse, if $L_{a_P}  \in \partial \Lbar(P)$ then, 
	$$
	\Lbar(P) = \ip{P}{L_{a_P} } = \min_{a \in A} \ip{P}{L_{a} },
	$$
	meaning $a$ is Bayes.

\end{proof}
Therefore, once non-Bayes actions are discarded, we can identify a loss with a subset of $\partial \Lbar$. Rather than working with a subset $\partial \Lbar$, it is advantageous to consider \emph{all} of $\partial \Lbar$.

\begin{definition}[Canonical Loss (Preliminary)]
	Let $\Lbar : \PP^+(\Theta) \rightarrow \RR$ be a concave 1-homogeneous function. Then its \emph{canonical loss}, $\Lagrange : \Theta \times \partial \Lbar \rightarrow \RR$ is given by $\Lagrange(\theta, \zeta) = \zeta(\theta)$.
	
\end{definition}
As will be shown, canonical losses can always be convexified. Furthermore, they maintain all the properties of $L$ needed to assess the quality of decisions.

\subsubsection{The Bayes Super-Prediction Set}

The process of \emph{canonising} a loss, that is, going from
$$
L \rightarrow \Lbar \rightarrow \Lagrange,
$$
can create \emph{extra} partial losses/actions that were not originally available to the decision maker. However, they do not get benefit from these extra actions. From any entropy define the \emph{Bayes super prediction set},
$$
\mathcal{S}_{\Lbar} := \left\{\zeta \in \RR^\Theta : \ip{\mu}{\zeta} \geq \Lbar(\mu) ,\ \forall \mu \in \RR^{\Theta}_+ \right\}. 
$$
By the definition,
$$
\min_{a \in A} \ip{P}{L_a} = \min_{\zeta \in \mathcal{S}_{\Lbar}} \ip{P}{\zeta} ,\ \forall P \in \PP(\Theta).
$$
The Bayes super-prediction set is precisely those partial losses that the decision maker does not need to use over the actions available to them, regardless of the distribution $P$. The super-prediction set is convex. Furthermore, the Bayes actions for $\Lagrange$ are the lower boundary of the super prediction set.

\begin{lemma}\label{Dual form of Canonicalization}
	Let $\Lbar : \PP^+(\Theta) \rightarrow \RR$ be a concave 1-homogeneous function. Then $\zeta \in \partial \Lbar$ if and only if, 
	$$
	\ip{\mu}{\zeta} \geq \Lbar(\mu) ,\ \forall \mu \in \PP^+(\Theta),
	$$
	with equality holding for at least one $\mu$.
	
\end{lemma}
The proof is a straightforward application of 1-homogeneity and supergradients. 
\\
\\
Canonical losses use \emph{all} supergradients of $\Lbar$. Proper losses use some.

\begin{corollary}[Loss from Entropy]
	Let $\Lbar : \PP^+(\Theta) \rightarrow \RR$ be a concave 1-homogeneous function and let $\nL : \PP^+(\Theta) \rightarrow \RR^\Theta$ be a super-gradient function, $\nL(\mu) \in \partial \Lbar(\mu) ,\ \forall \mu$. Then,
	$$
	L(\theta, Q) = \Lagrange(\theta,\nL(Q)),
	$$
	is a proper loss.
	
\end{corollary}

\subsection{Convexification of Losses in Canonical Form}

The preceding shows how to \emph{construct} losses, we begin with a concave 1-homogeneous function and take super-gradients. The focus now turns to their convexification. Once convexified, the decision maker has access to the large and ever-growing literature on the minimization of convex functions to aid in the calculation of optimal actions. We closely follow \citep{Dawid2007}, with a focus on canonical losses. This streamlines the development. For example, for some proper losses, lemma \ref{Convex set of Cans} fails to hold. Furthermore, our result on convexification of canonical losses (theorem \ref{Representation of Canonical Losses}), is to the best of our knowledge novel. 
\\
\\
Recall $\ind_\Theta \in \RR^\Theta$ is the function that always returns $1$, and define $\ind_\Theta^\bot$ to be its orthogonal complement in $\RR^\Theta$, the functions $v \in \RR^\Theta$ with,
$$
\ip{\ind_\Theta}{v} = \sum_{\theta \in \Theta} v(\theta) = 0.
$$ 
Define,
\begin{align*}
\Gamma_{\Lbar} &= \{(\gamma, v) \in \RR \times \ind_{\Theta}^{\bot}  : \gamma \ind_{\Theta} + v \in \partial \Lbar \ \}.
\end{align*}

\begin{lemma}\label{Canonical Coordinates}
	Let $(\gamma, v) \in \Gamma_{\Lbar}$. Then $\gamma$ is uniquely determined by $v$.
\end{lemma}

\begin{proof}
	Fix $v$ and suppose that there exist $\gamma_1$ and $\gamma_2$ with $\gamma_1 < \gamma_2$ and $\gamma_1 \ind_{\Theta} + v, \gamma_2 \ind_{\Theta} + v \in \partial \Lbar$. By assumption, $\gamma_2 \ind_{\Theta} + v$ is Bayes for some distribution $P$. But, 
	$$
	\ip{P}{\gamma_1 \ind_{\Theta} + v} = \gamma_1 + \ip{P}{v} < \gamma_2 + \ip{P}{v} = \ip{P}{\gamma_2 \ind_{\Theta} + v},
	$$
	a contradiction.
	
\end{proof}
Thus, we lose nothing by working with projections of losses onto $\ind_{\Theta}^{\bot}$. Define,
$$
\hat{\Gamma}_{\Lbar} = \mathrm{proj}_{\ind_{\Theta}^\bot}\left(\partial \Lbar \right)\subseteq \ind_{\Theta}^{\bot}.
$$
By lemma \ref{Canonical Coordinates} $\hat{\Gamma}_{\Lbar}$ is in 1-1 correspondence with $\partial \Lbar$. 

\begin{lemma}\label{Convex set of Cans}
	$\hat{\Gamma}_{\Lbar}$ is a convex set.
	
\end{lemma}

\begin{proof}
	To show that $\hat{\Gamma}_{\Lbar}$ is convex, we are required to show that for all $\zeta_1, \zeta_2 \in \partial \Lbar$ and all $\lambda \in [0,1]$ there is a constant $\gamma$ such that, 
	$$
	\lambda \zeta_1 + (1 - \lambda) \zeta_2 - \gamma \ind_{\Theta} \in \partial \Lbar.
	$$
	By lemma \ref{Dual form of Canonicalization}, this is equivalent to,
	$$
	\underbrace{\lambda \ip{P}{\zeta_1} + (1 - \lambda) \ip{P}{\zeta_2} - \Lbar(P)}_{\gamma(P)} - \gamma = \gamma(P) - \gamma \geq 0 ,\ \forall P \in \PP(\Theta),
	$$
	with equality holding for one $P$. Let $\gamma^* = \min_P \gamma(P)$, with $P^*$ the distribution that achieves the minimum. Clearly $\gamma(P) - \gamma^* \geq 0$. Therefore,
	$$
	\lambda \ip{P}{\zeta_1} + (1 - \lambda) \ip{P}{\zeta_2} - \gamma^* \geq \Lbar(P) ,\ \forall P \in \PP(\Theta),
	$$
	with equality for $P^*$. Therefore, by lemma \ref{Dual form of Canonicalization}, $\lambda \zeta_1 + (1 - \lambda) \zeta_2 - \gamma^* \ind_{\Theta} \in \partial \Lbar$.
	
\end{proof}
Define the function $\Psi : \hat{\Gamma}_{\Lbar} \rightarrow \RR$ such that, 
$$
v + \Psi(v) \ind_{\Theta} \in  \partial \Lbar.
$$ 
By lemma \ref{Canonical Coordinates}, $\Psi$ is well defined.

\begin{lemma}
	$\Psi$ is a convex function.
\end{lemma}

\begin{proof}
	Let $v_1, v_2 \in \hat{\Gamma}_{\Lbar}$ with $v_\lambda = \lambda v_1 + (1 - \lambda) v_2$. Let their partial losses be,
	\begin{align*}
	\zeta_1 &= v_1 + \Psi(v_1) \ind_{\Theta} \\
	\zeta_2 &= v_2 + \Psi(v_2) \ind_{\Theta} \\
	\zeta_\lambda &= \lambda v_1 + (1- \lambda) v_2 + \Psi(\lambda v_1 + (1- \lambda) v_2) \ind_{\Theta},
	\end{align*}
	respectively. By assumption, for all $\lambda \in [0,1]$ there exists a distribution $P_\lambda$ such that,
	$$
	\ip{P_\lambda}{\zeta_\lambda} \leq \ip{P_\lambda}{\zeta} ,\ \forall \zeta \in \partial \Lbar.
	$$
	Assume there is a $\lambda^*$ such that, 
	$$
	\lambda^* \Psi(v_1)  + (1- \lambda^*) \Psi(v_2) < \Psi(\lambda^* v_1 + (1- \lambda^*) v_2). 
	$$
	But then,
	$$
	\ip{P_{\lambda^*}}{\lambda^* \zeta_1 + (1- \lambda^*) \zeta_2} < \ip{P_{\lambda^*}}{\zeta_{\lambda^*}},
	$$
	a contradiction.
\end{proof}
This gives the following representation theorem for canonical losses.

\begin{theorem}[Representation of Canonical Losses]\label{Representation of Canonical Losses}
	Let $\Lbar : \PP^+(\Theta) \rightarrow \RR$ be a concave 1-homogeneous function. Then its canonical loss $\Lagrange$ can be represented as $\Lagrange : \Theta \times C  \rightarrow \RR$, with $C\subseteq \ind_{\Theta}^\bot$ a convex set and, 
	$$
	\Lagrange(\theta,v) = \ip{\delta_\theta}{v} + \Psi(v),
	$$ 
	for a convex function $\Psi$.
	
\end{theorem}

\section{Experiments}

We have assumed that the decision maker represents their uncertainty in $\Theta$ through a probability distribution $P$. The key question is \emph{which} $P$. Under the assumption that the decision maker aims to minimize their expected loss under $P$, this question is equivalent to asking which admissible action $a_P$ the decision maker should use. To determine $P$, the decision maker is guided by \emph{experiments}. 
\\
\\
In its simplest form, an experiment is a function $e : \Theta \rightarrow \PP(\obs)$, which assigns to each value $\theta$ a probability distribution on the results of the experiment. By linearity, this function can be extended to a \emph{linear} function $\dual{\Theta} \rightarrow \dual{\obs}$. These more general objects will be the focus of our study. 

\subsection{Transitions and their Algebra}

\begin{definition}
	A \emph{transition} from a set $X$ to a set $Y$ is a linear map $T : \dual{X} \rightarrow \dual{Y}.$
\end{definition}
Although abstract in appearance, we observe that when $X$ and $Y$ are \emph{finite}, a transition is nothing more than a \emph{matrix}. In general, a transition is an integral operator. Denote the set of all the transitions from $X$ to $Y$ by $\TT(X,Y)$. We call a transition \emph{Markov} if $T(\PP(X)) \subseteq \PP(Y)$. Denote the set of all Markov transitions from $X$ to $Y$ by $\MM(X,Y)$. When $X$ and $Y$ are finite, Markov transitions are represented by column stochastic matrices. The distribution,
$$
T(x) := T(\delta_x)
$$ 
is how the decision maker summarizes their uncertainty about $Y$ if the true value of $X$ is $x$. Every function $\phi \in Y^X$ defines a transition with,
$$
\ip{\phi(\alpha)}{f} := \ip{\alpha}{f \circ \phi}  ,\ \forall f \in \RR^Y, \ \forall \alpha \in \dual{X}.
$$
Such a transition is called \emph{deterministic}. Transitions can be combined in \emph{series} and in \emph{parallel}.
\\
\\
For transitions $f \in \TT(X,Y)$ and $g \in \TT(Y,Z)$ we can define $g\circ f\in \TT(X,Z)$ by the composition of functions. If $f_i \in \TT(X_i,Y_i)$, $i \in [1;k]$, are transitions then denote,
$$
\otimes_{i=1}^k f_i \in \TT(\times_{i=i}^k X_i, \times_{i=1}^k Y_i )
$$ 
with $\otimes_{i=1}^k f_i(x) = f_1(x_1) \otimes \dots \otimes f_k(x_k)$, where $\times$ denotes the Cartesian product and $\otimes$ denotes dual products. Transitions can also be \emph{replicated}. For any transition $f \in \TT(X, Y)$ we denote the \emph{replicated transition} $f_n \in \TT(X, Y^n) ,\ n\in \{1,2,\dots\}$, with, 
$$
f_n(x) := \underbrace{f(x) \otimes \dots \otimes  f(x)}_{n\text{ times}} := f(x)^n,
$$ 
the $n$-fold product of $f(x)$.
\\
\\
The above can be understood as defining a particular category, with objects as finite sets and arrows as transitions. Restricting to Markov transitions yields a subcategory, sometimes called the category of stochastic relations. This category will be of our primary interest.

\subsection{Comparing Experiments}\label{sec:comparing-experiments}
An \emph{experiment} is a Markov transition $e \in \MM(\Theta, \obs)$. We call $\obs$ the observation space of the experiment. We assume that $\obs$ is finite. The distribution $e(\theta)$ summarizes the uncertainty of the decision maker in the observation when $\theta$ is the unknown value. Leaving $\obs$ to vary yields a category $\MM(\Theta,-)$, of elements under $\Theta$. Objects of this category are experiments with arrows given by commutative triangles,	
	
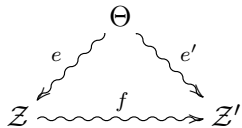
\[
	\xymatrix{
		& \Theta \ar@{~>}[ld]_e \ar@{~>}[rd]^{e'} & \\
		\obs\ar@{~>}[rr]^f && \obs'
	}
	\]
The Markov transition $f$ can be understood as adding noise to the experiment $e$, which results in the corrupted experiment $e'$. After observing the results of an experiment, the decision maker is tasked with choosing a suitable action. They do this through \emph{decision rules}.
\\
\\
A \emph{decision rule} is a Markov transition $d \in \MM(\obs, A)$. The distribution $d(z)$ summarizes the uncertainty of decision maker in which action to choose, given an observation $z \in \obs$. We define the \emph{risk},
$$
\riskL{L}(\theta, e, d) = \EE_{a \dist d\circ e(\theta)} L(\theta,a).
$$
The risk measures the quality of the final action chosen by the decision maker when using the decision rule $d$, after performing the experiment $e$, assuming $\theta$ is the true value of the unknown. The risk does not provide a single number for the comparison of experiments; rather, it provides an entire \emph{risk profile}. To compare risks directly, the decision maker can use \emph{Bayesian} or \emph{max} risks defined as
\begin{align*} 
\riskL{L}^{\pi}(e,d) := \EE_{\theta \dist \pi} \riskL{L}(\theta,e,d) \ \text{and}\ & \riskL{L}(e,d) := \max_{\theta} \riskL{L}(\theta,e,d), 
\end{align*}
respectively. Bayesian risk is more appropriate if the decision maker has some intuition about $\theta$, given in the form of a prior probability distribution $\pi$. The maximum risk is more appropriate if the decision maker has no prior knowledge of $\theta$.  These quantities allow the decision maker to \emph{compare} the usefulness of experiment, decision rule pairs. To compare experiments directly, we assume that the decision maker uses the \emph{best} decision rule. Define the minimum Bayesian risk and minimax risk as,
\begin{align*} 
\Rbar{L}^{\pi}(e) := \min_{d} \riskL{L}^{\pi}(e,d) \ \text{and}\ & \Rbar{L}(e) := \min_{d} \riskL{L}(e,d),
\end{align*}
respectively. The minimum Bayes risk and the minimax risk are deeply related.

\begin{theorem}
	For all experiments $e$ and loss functions $L$,
	$$
	\Rbar{L}(e) = \max_{\pi \in \PP(\Theta)} \Rbar{L}^{\pi}(e).
	$$
	
\end{theorem}
The proof is a simple application of the minimax theorem \citep{Komiya1988}. In light of this theorem, we focus on Bayesian risks for the remainder.

\subsection{Risk as Loss}

Leaving out the details of the experiment, what is important to the decision maker is the quality of the composed transition $d \circ e \in \MM(\Theta,A)$. One could call such a transition a \emph{strategy}. Risk is a function,
$$
\riskL{L} : \theta \times \MM(\Theta,A) \rightarrow \RR.
$$
with $\riskL{L}(\theta,s) = \EE_{a \dist s(\theta)} L(\theta, a)$. Risk therefore is a type of loss. Different experiments allow the decision maker access to different subsets of $\MM(\Theta,A)$. As will be seen, understanding these subsets and how they compare will be key to understanding experiments. Of course, all of the tools of loss function representation will be available to us to perform these comparisons.

\subsection{Bias-Variance Decomposition for Risks}\label{Bias Variance}

When working with canonical losses, the risk admits a simple and understandable decomposition. For canonical losses, we have,

\begin{align*}
	\riskL{\Lagrange}(\theta, e, d) &= \EE_{Z} \Lagrange(\theta, d(Z)) \\
	&= \EE_Z \ip{\delta_\theta}{d(Z)} + \Psi(d(Z)) \\
	&= \ip{\delta_\theta}{\bar{d}} + \Psi(\bar{d}) +\EE_Z \Psi(d(Z)) - \Psi(\bar{d}) \\
	&= \underbrace{\Lagrange(\theta,\bar{d})}_{Bias} + \underbrace{\EE_Z \Psi(d(Z)) - \Psi(\bar{d})}_{Variance},
\end{align*}
where $Z \dist e(\theta)$, and $\bar{d}$ is the average action selected by $d$.

\subsection{Admissible and Bayesian decision rules}

The optimal decision rule will, in general, depend on the prior knowledge of the decision maker about the unknown. Even without this knowledge, the decision maker can remove rules that are obviously not optimal.

\begin{definition}
	Let $e$ be an experiment. A decision rule $d$ is \emph{admissible for $e$} if there does not exist a decision rule $d'$ with,
	$$
	\riskL{L}(\theta, e, d') \leq \riskL{L}(\theta, e, d) ,\ \forall \theta \in \Theta
	$$
	with strict inequality for at least one $\theta$.
	
\end{definition}
A decision rule is admissible if it is not obviously worse than some other decision rule. If the decision maker has prior knowledge $\pi$, they can minimize the Bayesian risk using a Bayesian decision rule.

\begin{definition}
	Let $e$ be an experiment and $\pi$ be a prior. A decision rule $d^*$ is \emph{Bayes for $(\pi,e)$} if,
	$$
	d^* \in \argmin_{d} \riskL{L}^{\pi}(e,d).
	$$
\end{definition} 
Much like the case for Bayesian actions, the decision maker need only consider Bayesian decision rules.

\begin{theorem}[Complete Class Theorem \citep{Wald1949}]
	A decision rule $d$ is admissible for $e$ if and only if there exists a prior $\pi$ such that $d$ is Bayes for $(\pi, e)$.
\end{theorem}
The above theorem says that Bayesian algorithms provide all the rules that a sensible decision maker should use. Picking a particular admissible algorithm is \emph{equivalent} to picking a prior $\pi$ and minimizing the Bayesian risk against that prior. Although statistically speaking, admissible algorithms afford no obvious improvements, they may be hard to implement. Our language does not take this into account. The study of inadmissible algorithms and their risks is, therefore, a worthwhile endeavor. 
\\
\\
Bayes optimal algorithms admit a simple representation. No paper on decision theory would be complete without reference to Bayes' rule. In our language, Bayes' rule provides means to \emph{reverse} an experiment. Let $e(\pi)$ be the marginal distribution over the observation space, and $e^\dagger \in \MM(\obs,\Theta)$ be the induced Markov transition of unknowns given observations obtained via Bayes' rule. Then,
$$
\Rbar{L}^\pi(e) =\EE_{z \dist e(\pi)}\Lbar(e^\dagger(z)),  
$$
with Bayes optimal decision rule,
$$
d(z) = \argmin_{a \in A} \EE_{\theta \dist e^\dagger(z)} L(\theta,a).
$$
$d(z)$ proceeds first by \emph{inferring} which $\theta$ is true, before acting optimally. Denote by $\bullet$ the canonical set of one element. To include prior distributions in our language, we can work with the category of elements under $\bullet$, denoted by $\MM(\bullet,-)$. Objects of this category are precisely prior distributions defined over the relevant set. Arrows are commutative triangles that ensure the priors match up, ie an arrow,
	
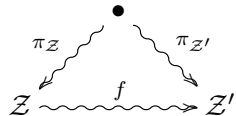
\[
	\xymatrix{
		& \bullet \ar@{~>}[ld]_{\pi_\obs} \ar@{~>}[rd]^{\pi_{\obs'}} & \\
		\obs\ar@{~>}[rr]^{f} && \obs'
	}
	\]
ensures $\pi_{\obs'} = f(\pi_\obs)$. Bayes' rule then asserts that this category is a dagger category, to each arrow $f$ we obtain a reversed arrow.

\section{When is One Experiment Always Better than Another?}\label{sec:when-is-one-experiment-always-better-than-another?}

Let $e$ and $e'$ be experiments. Suppose that due to constraints, the decision maker can only perform one of these two experiments. The decision maker can compare the Bayes or minimax risks of the two experiments, however, this involves performing a calculation. Furthermore, if the loss function of interest changes, then the ordering of the experiments might change. We seek qualitative results in relation to when $e$ is \emph{always} better than $e'$ no matter what the loss or prior distribution.

\begin{definition}
	Let $e \in \MM(\Theta, \obs)$ and $e' \in \MM(\Theta, \obs')$ be experiments. $e$ \emph{divides} $e'$ (written $e \mid e'$) if there exists a Markov transition $f \in \MM(\obs, \obs')$ such that $e' = f \circ e$.
	
\end{definition}
$e \mid e'$ if $e'$ is $e$ with extra noise $f$. We make this intuition precise with theorem \ref{BSS Theorem}. For an experiment $e$, let $\MM_{e}(\Theta, A)$ be the set of transitions from $\Theta$ to $A$ that $e$ divides.

\begin{theorem}
	$e$ divides $e'$ if and only if for all action sets $A$, 
	$$
	\MM_{e'}(\Theta, A) \subseteq  \MM_{e}(\Theta, A).
	$$
\end{theorem}

\begin{proof}
	The forward implication follows simply from the definition. For the converse, take $A = \obs'$ and note $e' \in \MM_{e'}(\Theta,\obs')$. By assumption, 
	$$
	\MM_{e'}(\Theta, \obs') \subseteq  \MM_{e}(\Theta, \obs').
	$$
	As $e' \in \MM_{e'}(\Theta, \obs')$, this implies that there exists a $f$ with $f \circ e = e'$.
	
\end{proof}

\subsection{The Blackwell-Sherman-Stein Theorem and Sufficiency}

\begin{theorem}[Blackwell-Sherman-Stein Theorem \citep{Blackwell:1954}]\label{BSS Theorem}
	Let $e$ and $e'$ be experiments. $e \mid e'$ if and only if for all action sets, loss functions and priors, 
	$$
	\Rbar{L}^{\pi}(e) \leq \Rbar{L}^{\pi}(e'). 
	$$
\end{theorem}
We prove the forward implication, called the data processing theorem. The proof of the converse will come later as a simple corollary of the randomization theorem.
\\
\begin{proof}	
	For any decision rule $d' \in \MM(\obs', A)$ consider the decision rule $d = d' \circ f \in \MM(\obs, A)$. As $e' = f \circ e$, it is easy to verify that,
	$$
	\riskL{L}^\pi(e', d') = \riskL{L}^\pi(f \circ e, d') = \riskL{L}^\pi(e, d' \circ f) = \riskL{L}^\pi(e, d).
	$$
	To complete the proof take minima over $d$ and $d'$.
	
\end{proof}
We say $e$ and $e'$ are \emph{equivalent} experiments (written $e \cong e'$) if both $e \mid e'$ and $e' \mid e$. Equivalent experiments have equivalent risks. A key notion in statistics is that of \emph{sufficiency}. A \emph{sufficient statistic} is a function of the observation that loses none of the information contained in $e$. Identifying and exploiting sufficient statistics allows the decision maker to compress the information contained in the observation without losing information.

\begin{definition}
	Let $e \in \MM(\Theta, \obs)$ be an experiment. A Markov transition $f \in \MM(\obs, \tilde{\obs})$ is \emph{sufficient} for $e$ if $f \circ e \cong e$.
\end{definition}
By the Blackwell-Sherman-Stein theorem, sufficient statistics maintain all information in the observation.

\subsubsection{Sufficiency and Bayes' Rule}

Given a prior, Bayes theorem provides means to reverse an experiment and provide a transition $e^\dagger$ that \emph{infers} the correct $\theta$. Let $e' = e^\dagger \circ e$. Clearly,
$$
\Rbar{L}^{\pi}(e) = \Rbar{L}^{\pi}(e').
$$
To see this, remember that the decision rule that minimizes the Bayes risk is precisely that which infers $\theta$ through $e^\dagger$ and then acts optimally.

\subsection{Most and Least Informative Experiments}

For any set $\Theta$ of unknowns, there is a \emph{most} informative and a \emph{least} informative experiment. Recall the identity function $\id{\Theta}$, $\id{\Theta}(\theta) = \theta$. For any experiment $e$, we have $e \circ \id{\Theta} =  e$. Therefore, $\id{\Theta}$ divides any experiment. $\id{\Theta}$ provides the decision maker the \emph{exact} value of $\theta$. This experiment has risk, 
$$
\Rbar{L}^\pi(\id{\Theta}) = \EE_{\theta \dist \pi} \min_{a \in A} L(\theta,a).
$$
For any set $X$, define the \emph{terminal} transition $\bullet_X \in\MM(X, \{1\})$ with $\bullet_{X}(x) = 1$ for all $X$. This transition eliminates all information about $X$. Much like the identity transition divides every experiment, the terminal transition is divided by every experiment. For all experiments $e$, $\bullet_{\Theta} = \bullet_{\obs} \circ e$. 
This experiment has risk,
$$
\Rbar{L}^\pi(\bullet_{\Theta}) = \Lbar(\pi). 
$$
By the data processing theorem,
$$
\Rbar{L}^\pi (\id{\Theta}) \leq \Rbar{L}^\pi ( e) \leq \Rbar{L}^\pi (\bullet_{\Theta}).
$$
Assume that $\min_{a \in A} L(\theta,a) = 0, \forall \theta \in \Theta$. Then,
$$
0 \leq \frac{\Rbar{L}^\pi (e) }{\Lbar(\pi)} \leq 1.
$$

\subsection{The General Data Processing Theorem}

The key to the proof of the data processing theorem is the insight that if $e \mid e'$ then,
$$
\MM_{e'}(\Theta, A) \subseteq  \MM_{e}(\Theta, A).
$$
This means that the decision maker has more decision rules after performing the experiment $e$ than $e'$, since one could always first inject noise and then proceed as if experiment $e'$ had been performed. This suggests another means to construct quantities that satisfy the data processing theorem.

\begin{theorem}\label{Generalized Information Processing}
	Let $\riskL{} : \MM(\Theta, A) \rightarrow \RR$ and define,
	$$
	\Rbar{}(e) = \min_{d \in \MM_{e}(\Theta, A)} \riskL{}(d).
	$$
	If $e \mid e'$ then $\Rbar{}(e) \leq \Rbar{}(e')$.
	
\end{theorem}

\begin{proof}
	If $e \mid e '$ then $\MM_{e'}(\Theta, A) \subseteq  \MM_{e}(\Theta, A)$. Therefore,
	$$
	\min_{d \in \MM_{e}(\Theta, A)} \riskL{}(d) \leq \min_{d' \in \MM_{e'}(\Theta, A)} \riskL{}(d').
	$$
	
\end{proof}
We recover the usual data processing theorem by taking, 
$$
\riskL{}(d) = \EE_{\theta \dist \pi} \EE_{a\dist d(\theta)}L(\theta, a).
$$ 
Remarkably, the proof of theorem \ref{Generalized Information Processing} makes no explicit reference to expected risks, transitions, or even probability distributions. One can understand data processing theorems as functors. Define the category $\left(\RR,\leq\right)$, with objects real numbers and arrows $x \rightarrow y$ if and only if $x\leq y$. Consider the category of finite sets with arrows Markov transitions. Fix a set of unknowns $\Theta$ and consider experiments on $\Theta$, elements of $\MM(\Theta,-)$. This is naturally a category, sometimes called a slice category or the category of elements under $\Theta$, with arrows given by commutative triangles linking experiments via noise,
	
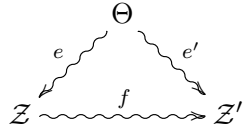
\[
	\xymatrix{
		& \Theta \ar@{~>}[ld]_e \ar@{~>}[rd]^{e'} & \\
		\obs\ar@{~>}[rr]^f && \obs'
	}
\]
Then the construction of Theorem \ref{Generalized Information Processing} can be understood as a functor from $\MM(\Theta,-)$ to $\left(\RR,\leq\right)$.

\subsection{Understanding the Blackwell-Sherman-Stein Theorem as an Equivalence of Categories}

Fix $\Theta$ and consider the category of experiments on $\Theta$. Construct the category $\riskL{}(\Theta)$ as follows:

\begin{itemize}
	\item The objects of $\riskL{}(\Theta)$ are the objects of $\MM(\Theta,-)$, namely experiments.
	\item An arrow from $e \rightarrow e'$ is the relationship that $\Rbar{L} (e) \leq \Rbar{L}(e')$ for all bounded loss functions.
\end{itemize}
Then the Blackwell-Sherman-Stein theorem states that these categories are equivalent $\riskL{}(\Theta) \cong \MM(\Theta,-)$.

\subsection{An Example: Binary Experiments and the Variational Divergence}

Let $\Theta = A =\{-1,1\}$ with $L(\theta,a) = \pred{\theta \neq a}$. In words, there are two unknowns, and all the loss function cares about is that we accurately identify which of the two it is. Let $e$ be an experiment and $\pi$ the uniform prior on $\Theta$. Then,
\begin{align*}
	\Rbar{L}^\pi(e) &= \min_{f\in \{-1,1\}^\obs} \EE_{\theta \dist \pi} \EE_{x\dist e(\theta)}L(\theta,f(x)) \\
	&= \min_{f\in \{-1,1\}^\obs} \frac{1}{2} \left(\EE_{e(-1)}\pred{1 = f(Z)} + \EE_{e(1)} \pred{-1 = f(Z)}\right)
\end{align*}
Let $A = \{z \in \obs : f(z) = 1\}$, then
\begin{align*}
	\Rbar{L}^\pi(e) &= \frac{1}{2} \min_{A} \left(e_{-1}(\ind_A) + e_{1}(\ind_{A^c}) \right) \\
	& = \frac{1}{2} \min_{A} \left(e_{-1}(\ind_A) + 1 - e_{1}(\ind_{A}) \right).
\end{align*}
This yields,
\begin{align*}
	\Rbar{L}(\bullet_{\Theta}) - \Rbar{L}^\pi(e) &= \frac{1}{2} - \frac{1}{2} \min_{A} \left(e_{-1}(\ind_A) + 1 - e_{1}(\ind_{A}) \right) \\
	&= \frac{1}{2} \max_{A} \left(e_{-1}(\ind_A) - e_{1}(\ind_{A}) \right) \\
	&= V(e_{-1},e_1),
\end{align*}
which is the total variational distance between $e(-1)$ and $e(1)$, a commonly used metric between probability distributions. The data processing theorem recovers the well-known fact that,
$$
V(P,Q) \geq V(f(P),f(Q)),
$$
for all distributions $P,Q \in \PP(\obs)$ and all Markov kernels $f \in \MM(\obs,\obs')$.

\subsection{Example: General Experiments and the Mutual Information}

Here $\Theta$ is a general finite set and $A = \PP(\Theta)$. Recall the log loss, also known as the negative likelihood,
$$
L(\theta,Q) = -\log\left(Q(\theta)\right).
$$
This loss is proper,
$$
P = \argmin_{Q \in \PP(\Theta)} \EE_{P} L(\Theta,Q).
$$
The corresponding entropy function is the Shannon entropy,
$$
\Lbar(P) = -\EE_{P}\log\left(P(\Theta)\right) = H(P).
$$
Now let $e$ be an experiment. We have,
\begin{align*}
	\Rbar{L}^\pi(e) &= \EE_{z\dist e(\pi)} \Lbar(e^\dagger(z)) \\
	&= \EE_{z\dist e(\pi)} H(\Theta | \obs = z) \\
	&= H(\Theta|\obs).
\end{align*}
Then,
\begin{align*}
	\Rbar{L}(\bullet_{\Theta}) - \Rbar{L}^\pi(e)  &= H(\Theta) - H(\Theta | \obs) \\
	&= I(\Theta,\obs)
\end{align*}
The mutual information between $\Theta$ and $\obs$. Our general data processing theorem then recovers the well known fact that if,
$$
\Theta \rightarrow \obs \rightarrow \obs',
$$
is a Markov chain,
$$
I(\Theta,\obs) \geq I(\Theta,\obs').
$$

\subsection{Connections to \texorpdfstring{$\phi$}{phi}-divergences and multi-\texorpdfstring{$\phi$}{phi}-divergences.}

Let $\phi : [0,\infty) \rightarrow \RR$ be a convex function with $\phi(1)=0$. Then the $\phi$-divergence between the distributions $P,Q \in \PP(\obs)$ is given by,
$$
D_\phi(P,Q) = \EE_{P}\phi\left(\frac{dQ}{dP}\right),
$$
if $Q \ll P$ and $\infty$ otherwise. As we have seen, for binary experiments, for a particular loss and prior, the gap,
$$
\Rbar{L}(\bullet_{\Theta}) - \Rbar{L}^\pi(e) = V(e_{-1},e_{1}).
$$
The variational divergence is an example of a $\phi$-divergence with $\phi(x) = |x - 1|$. This construction works more broadly, for binary experiments the gap,
$$
\Rbar{L}(\bullet_{\Theta}) - \Rbar{L}^\pi(e) = D_\phi(e_{-1},e_{1}),
$$
for some $\phi$. Similarly, for every $\phi$-divergence, there is a loss and prior that makes the above hold. We refer the reader to \citep{Reid2009b} for the details of this correspondence. Our general data processing theorem then recovers the fact,
$$
D_\phi(P,Q) \geq D_\phi (f(P),f(Q)).
$$
Similarly, if $|\Theta| > 2$ we can invoke a similar correspondence between,
$$
\Rbar{L}(\bullet_{\Theta}) - \Rbar{L}^\pi(e)
$$
and objects called multi-$\phi$-divergences. We refer the reader to \citep{Garca-Garca}. Our general data processing theorem once again recovers the data processing theorem for multi-$\phi$-divergences.

\section{Deficiency and Quantitative Data Processing Theorems}

The converse of the Blackwell-Sherman-Stein theorem states that if $e$ does \emph{not} divide $e'$ then there is a loss function and prior that renders $e'$ more useful. The gap in risks is quantified by \emph{deficiency}.

\begin{definition}
	Let $e \in \MM(\Theta, \obs)$ and $e' \in \MM(\Theta, \obs')$ be experiments. The \emph{directed deficiency} from $e$ to $e'$ is,
	$$
	\xi^{\pi}(e, e') := \min_{f \in \MM(\obs, \obs')} \EE_{\theta \dist \pi} V(f\circ e(\theta), e'(\theta)).
	$$ 
\end{definition}
The directed deficiency provides means to quantify how close $e$ is to dividing $e'$. $\xi^{\pi}(e, e')= 0$ for all priors if and only if $e \mid e'$. The \emph{deficiency} is defined as,
$$
\Xi^{\pi}(e,e') := \max\{\xi^{\pi}(e, e'), \xi^{\pi}(e', e)\}.
$$
Deficiency measures how close to equivalence $e$ and $e'$ are. $\Xi^{\pi}(e,e') = 0$ for all priors if and only if $e \cong e'$. The directed deficiency provides a \emph{quantitative} version of the Blackwell-Sherman-Stein theorem. 

\begin{theorem}[Randomization Theorem \citep{Le1964}]
	Fix $\epsilon >0$ and a prior $\pi$. Let $e$ and $e'$ be experiments. Then,
	$$
	\Rbar{L}^\pi(e) \leq \Rbar{L}^\pi(e') + \epsilon \norm{L}_\infty
	$$ 
	for all action sets and loss functions, if and only if $\xi^{\pi}(e,e') \leq \epsilon$.
	
\end{theorem}
We present the proof from \citep{Torgersen1991}, with some streamlining.

\begin{proof}	
	We begin with the reverse implication. As $\xi^{\pi}(e,e') \leq \epsilon$, there exists a transition $f \in \MM(\obs,\obs')$ such that,
	$$
	\EE_{\theta \dist \pi} V(f\circ e(\theta), e'(\theta)) \leq \epsilon.
	$$
	Now fix a decision rule $d' \in \MM(\obs',A)$, and consider $d = d' \circ f$ as in the diagram below.	
	
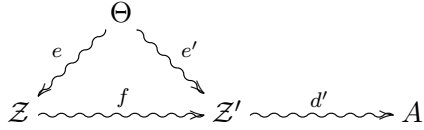
\[
		\xymatrix{
			& \Theta \ar@{~>}[ld]_e \ar@{~>}[rd]^{e'} & &&  \\
			\obs\ar@{~>}[rr]^f && \obs' \ar@{~>}[rr]^{d'} && A
		}
	\]
	We have,
	\begin{align*}
	\riskL{L}^\pi(e, d) - \riskL{L}^\pi(e', d') &= \EE_{\theta \dist \pi} \left[ \EE_{a \dist d \circ e(\theta)} L(\theta,a) -  \EE_{a \dist d' \circ e'(\theta)} L(\theta,a) \right]\\
	&\leq \EE_{\theta \dist \pi}  V(d \circ e(\theta) , d' \circ e'(\theta) ) \norm{L}_\infty \\
	&=  \EE_{\theta \dist \pi} V(d' \circ f \circ e(\theta) , d' \circ e'(\theta) )  \norm{L}_\infty\\
	&\leq \EE_{\theta \dist \pi} V(f \circ e(\theta) , e'(\theta) )  \norm{L}_\infty \\
	&\leq \epsilon \norm{L}_\infty
	\end{align*}
	where the first line follows from the definition of the Bayesian risk, the second follows from the definition of the variational distance, the third from the definition of $d$, the fourth as variational distance is an $\phi$-divergence and therefore satisfies a data processing inequality and finally from our assumptions on $f$. The proof is completed by taking a minimum over $d'$ and $d$. 
	\\
	\\
	For the forward implication, first fix a set of actions $A$ and a decision rule $d' \in \MM(\obs',A)$ and define the function,
	$$
	\phi(L,d) = \riskL{L}^\pi(e, d) - \riskL{L}^\pi(e', d') - \epsilon \norm{L}_\infty.
	$$
	Note that $\phi$ is affine in $d$ and concave in $L$. By the conditions in the theorem,
	$$
	\sup_{L} \min_{d} \phi(L,d) \leq 0.
	$$
	By the minimax theorem \citep{Komiya1988} or strong convex duality \citep{Lucchetti2006}, there exists a saddle point $(L^*, d^*)$ with, 
	$$
	\phi (L^*, d^*) = \min_{d} \sup_{L} \phi(L,d) = \sup_{L} \min_{d} \phi(L,d) \leq 0  .
	$$
	This implies,
	$$
	\riskL{L}^\pi(e, d^*) \leq \riskL{L}^\pi(e', d') + \epsilon \norm{L}_\infty ,\ \forall L.
	$$
	This means $\EE_{\theta \dist \pi} V(d^* \circ e(\theta) , d' \circ e'(\theta) ) \leq \epsilon$, from the definition of variational distance. Note that $d'$ and the action set $A$ are arbitrary. To complete the proof, take $A = \obs'$ and $d' = \id{\obs'}$. The transition $f$ is given by $d^*$.
	
\end{proof}
The proof of the reverse implication of the Blackwell-Sherman-Stein theorem can be recovered by setting $\epsilon = 0$. The randomization theorem shows that there is a deep connection between differences in risk and deficiency. The following theorem makes this connection precise. 

\begin{theorem}\label{Deficiency as a supremum over all losses}
	Let $e$ and $e'$ be experiments. For all priors $\pi$, 
	$$
	\Xi^{\pi}(e,e') = \sup_{L : \norm{L}_\infty \neq 0} \frac{\left|\Rbar{L}^{\pi}(e) - \Rbar{L}^{\pi}(e')\right|}{\norm{L}_\infty}.
	$$
\end{theorem}
For the proof, we require the following simple lemma.

\begin{lemma}
	For $x,y \in \RR$ if $\forall \epsilon \in \RR$, $x \leq \epsilon \Leftrightarrow y \leq \epsilon$ then $x = y$.
\end{lemma}

\begin{proof}
	Suppose that $x \neq y$ and without loss of generality assume that $x < y$. Set $\epsilon = \frac{x + y}{2}$. Then $x \leq \epsilon$ and $y > \epsilon$, which implies the contrapositive.
\end{proof}
We now prove the theorem.

\begin{proof}
	If $\Xi^{\pi}(e,e') \leq \epsilon$ then $\xi^{\pi}(e,e') \leq \epsilon$ and $\xi^{\pi}(e',e) \leq \epsilon$. By the randomization theorem,
	$$
	\frac{\left| \Rbar{L}^{\pi}(e) - \Rbar{L}^{\pi}(e') \right| }{\norm{L}_\infty}\leq \epsilon  ,\ \forall L : \norm{L}_\infty \neq 0.
	$$
	Conversely, if, 
	$$
	\sup_{L : \norm{L}_\infty \neq 0} \frac{\left|\Rbar{L}^{\pi}(e) - \Rbar{L}^{\pi}(e')\right|}{\norm{L}_\infty}\leq \epsilon,
	$$
	then $\Rbar{L}^{\pi}(e) \leq \Rbar{L}^{\pi}(e') + \epsilon \norm{L}_\infty$ and $\Rbar{L}^{\pi}(e') \leq \Rbar{L}^{\pi}(e) + \epsilon \norm{L}_\infty$. By the randomization theorem, this means $\Xi^{\pi}(e,e') \leq \epsilon$. This, combined with the above lemma, completes the proof.
	
\end{proof}
The randomization theorem can be used to define \emph{quantitative} versions of concepts such as sufficiency. This was Le Cam's original motivation for defining the quantity in \citep{Le1964}. $f$ is approximately sufficient for $e$ if $\Xi^\pi(e, f\circ e)$ is small. Deficiency provides a metric on experiments.

\begin{theorem}
	Let $\pi$ be a prior that assigns a nonzero probability to each unknown. Then $\Xi^{\pi}$ is a metric on experiments modulo equivalence.
\end{theorem}

\begin{proof}
	$\Xi^{\pi}$ is obviously nonnegative and symmetric. We are required to show that it satisfies the triangle inequality. Let $e, e', e''$ be experiments, with $f$ and $f'$ transitions as in the diagram below.
	
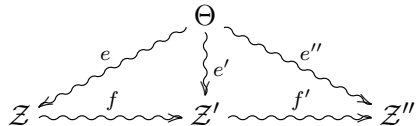
\[
		\xymatrix{
			&& \Theta \ar@{~>}[lld]_{e} \ar@{~>}[d]^{e'} \ar@{~>}[rrd]^{e''} &&\\
			\obs \ar@{~>}[rr]^{f} && \obs' \ar@{~>}[rr]^{f'} && \obs''
		}
	\]
	We have for all $\theta$,
	\begin{align*}
	V(e''(\theta) , f' \circ f \circ e(\theta)) &\leq V(e''(\theta) , f' \circ e'(\theta)) + V(f' \circ e'(\theta) , f' \circ f \circ e(\theta)) \\
	&\leq V(e''(\theta) , f' \circ e'(\theta))  + V(e'(\theta) , f \circ e(\theta)),
	\end{align*}
	where we have used the fact that the variational distance is a metric, followed by the data processing inequality. Averaging over $\pi$ and taking a minimum over $f$ and $f'$ yields,
	$$
	\xi^{\pi}(e, e'') \leq \xi^{\pi}(e, e') + \xi^{\pi}(e', e'').
	$$
	Reversing the direction and taking maximums yields the desired result.
\end{proof}

\subsection{Calculating Deficiency}

The variational distance can be calculated as the $l_1$ distance,
$$
V(P,Q) = \frac{1}{2}\sum_{z \in \obs} \left| P(z) - Q(z) \right|.
$$
Experiments $e \in \MM(\Theta, \obs)$ can be represented by a $|\obs| \times |\Theta|$ column stochastic matrix. Furthermore, the prior distribution $\pi$ can be represented by a vector in $\RR^{|\Theta|}$. Using these representations, the directed deficiency can be calculated using linear programming.

\begin{lemma}\label{Calculating Deficiency}
	
	Let $e$ and $e'$ experiments with their stochastic matrix representation given by $E$ and $E'$ respectively. Then $\xi^{\pi}(e, e')$ can be calculated via the following linear program,
	\begin{align*}
	& \min_{M_{ij}, F_{ij}}  \sum\limits_{i = 1}^{|\obs'|} \sum\limits_{j = 1}^{|\Theta|} M_{ij} \\
	& \text{subject to} \\
	& M_{ij}, F_{ij} \geq 0\ \text{and} -M_{ij} \leq \pi_j E'_{ij} -  \pi_j \left[F E\right]_{ij} \leq M_{ij} \ \forall i, j \\
	& \sum\limits_{i=1}^{|\obs'|}F_{ij} = 1 \ \forall j,
	\end{align*}
	where $\left[F E\right]_{ij}$ is the $ij$ entry of $T E$. 
	
\end{lemma}

\begin{proof}
	The constraints $F_{ij} \geq 0$ and $\sum\limits_{i=1}^{|\obs'|}F_{ij} = 1 \ \forall j$, ensure that $F$ is a stochastic matrix. Taking the final constraint and summing over $i$ and $j$ yields,
	\begin{align*}
	\sum\limits_{i = 1}^{|\obs'|} \sum\limits_{j = 1}^{|\Theta|} M_{ij} &\geq \sum\limits_{i = 1}^{|\obs'|} \sum\limits_{j = 1}^{|\Theta|} \left| \pi_j E'_{ij} -  \pi_j \left[F E\right]_{ij} \right| \\
	&= \sum\limits_{j = 1}^{|\Theta|}  \pi_j \sum\limits_{i = 1}^{|\obs'|}  \left| E'_{ij} - \left[F E\right]_{ij} \right| \\
	&= \EE_{\theta \dist \pi} V(e'(\theta), f\circ e(\theta)).
	\end{align*}
	Equality is attained in the above if $M_{ij} = \left| \pi_j E'_{ij} -  \pi_j \left[F E\right]_{ij} \right|$. Minimizing $M$ and $F$ produces an optimal solution of $\xi^{\pi}(e, e')$.
	
\end{proof}

\section{Conclusion}

This paper delivers a powerful unifying lens on statistical experiments by treating them as morphisms in a category of Markov transitions. It sharpens classic results and opens new avenues:

\begin{itemize}
	\item  Recasts experiments and decision rules in a single algebraic framework, revealing the essence behind Blackwell’s and Le Cam’s theorems. 
	\item  Introduces a general data processing inequality that extends beyond $\phi$-divergences to arbitrary risk functionals. 
	\item Provides a convex analytical treatment of losses via canonical reparameterization and neat bias-variance decomposition. 
	\item Turns Le Cam’s abstract deficiency into a concrete linear program.  
\end{itemize}

\newpage

\bibliographystyle{plainnat}
\bibliography{./References}

\end{document}